\pdfoutput=1
\documentclass{article}

     \usepackage[preprint,nonatbib]{neurips_2019}

\usepackage[utf8]{inputenc} 
\usepackage[T1]{fontenc}    
\usepackage{hyperref}       
\usepackage{url}            
\usepackage{booktabs}       
\usepackage{amsfonts}       
\usepackage{nicefrac}       
\usepackage{microtype}      
\usepackage{xspace}
\usepackage{graphicx}
\usepackage{subfigure}
\usepackage{algorithm}
\usepackage{algorithmic}
\usepackage{multirow}
\usepackage[numbers]{natbib}

\usepackage{amsmath,amsfonts,amsthm,bm}
\usepackage{bbm}

\newcommand{\comment}[1]{}

\newtheorem{theorem}{Theorem}
\newtheorem{lemma}{Lemma}

\def\eqref#1{equation~\ref{#1}}

\def\1{\mathbbm{1}}

\DeclareMathAlphabet{\mathsfit}{\encodingdefault}{\sfdefault}{m}{sl}
\SetMathAlphabet{\mathsfit}{bold}{\encodingdefault}{\sfdefault}{bx}{n}

\newcommand{\E}{\mathbb{E}}

\DeclareMathOperator*{\argmin}{arg\,min}

\newtheorem{definition}{Definition}

\theoremstyle{definition}

\usepackage{hyperref}

\title{Minimum-Margin Active Learning}

\author{%
  Heinrich Jiang\\
  Google Research\\
  Mountain View, CA\\
  \texttt{heinrichj@google.com} \\
  \And
  Maya Gupta\\
  Google Research\\
  Mountain View, CA\\
  \texttt{mayagupta@google.com} \\
}

\begin{document}

\maketitle

\begin{abstract}
We present a new active sampling method we call \emph{min-margin} which trains multiple learners on bootstrap samples and then chooses the examples to label based on the candidates' minimum margin amongst the bootstrapped models. This extends standard margin sampling in a way that increases its diversity in a supervised manner as it arises from the model uncertainty. We focus on the one-shot batch active learning setting, and show theoretically and through extensive experiments on a broad set of problems that min-margin outperforms other methods, particularly as batch size grows.   
\end{abstract}
\section{Introduction}

In many practical applications, the learner has access to an insufficient training set of labeled examples, and a large pool of unlabeled examples that could be labeled,  but labeling examples is costly. The goal of batch active learning is to select the most useful batch of examples to label given a budget that dictates the size of the allowed batch. Finding more effective batch active sampling approaches is ever more critical because machine learning is revolutionizing decision making systems.  Better batch active sampling is one of the most promising ways to make machine learning cheaper to improve. 

One of the most successful strategies for active sampling is to select candidates that the model is most uncertain about: \emph{uncertainty sampling} \cite{Lewis:1994}. In the context of discriminant-based classifiers like logistic regression or neural networks, this principle can be expressed as selecting candidates closest to the decision boundary, known as \emph{margin sampling} \cite{Scheffer:2001}. Margin sampling has no hyperparameters to tune, and is a well-regarded method that has proven difficult to beat experimentally  \cite{Reyes:2018,ScheinUngar:2007,Bilgic:2016,FuZhuLi:Survey2013,Tuia:2011}. However, when we select a large batch of candidates to label at once from a very large pool of possible candidates, margin sampling may not pick a diverse enough batch, and we find this lack of diversity grows worse with larger batch sizes or larger candidate pools.

Our proposed algorithm extends the core idea of margin sampling to a set of bootstrap models to increase the diversity of the selected batch. Two key differences to prior work is that our notion of diversity is \emph{supervised} by the model uncertainty, and that our algorithm \emph{scales linearly} in the batch size and candidate set size. We show through extensive experimentation that our algorithm consistently performs as well as margin sampling for smaller batches, and better than margin sampling for larger batches without any tuning of hyperparameters. Theoretical analysis confirms the trend that our strategy can provably dominate margin sampling for larger batch sizes.

Our proposal is related to Query By Committee (QBC) in that both use multiple models to determine which examples to label \cite{QueryByCommittee:1992}. A common practical version of QBC creates multiple models by using  bootstrap samples of the labeled training data \cite{Abe:1998}, as we do here.  Practical QBC algorithms generally score candidates based on some notion of disagreement of a fixed committee of classifiers \cite{Abe:1998,Dagan:1995,Copa:2010}. However, a recent thorough empirical survey
showed that despite its additional computational complexity, QBC did not provide any conclusive advantage over margin sampling \cite{Bilgic:2016}. Similarly, an experimental comparison of active sampling methods using SVM's on multiple remote sensing problems \cite{Tuia:2011} found margin sampling worked as well or better than the QBC variant considered (\emph{normalized entropy query by bagging} \cite{Copa:2010}). 

Many active sampling methods have been devised that directly try to increase the diversity of the candidates. A number of methods penalize candidates based on their similarity to candidates already selected for the batch  \cite{Zhou:2009,Brinker:2003,Fujii:1998,Ferecatu:2007}. These methods are computationally un-appealing for large datasets and batch sizes, as they scale at least $O(B^2|\mathcal{Z}|)$ in batch size $B$ and candidate set size $|\mathcal{Z}|$  while our proposal is $O(|\mathcal{Z}|)$. A later independent study \cite{Tuia:2011} did not show clearly better results from adding such diversity criteria \cite{Ferecatu:2007} to margin sampling. Submodularity can be used to jointly select a set of diverse points close to the margin \cite{Bilmes:2015,Hoi:2006}.  Other methods use an explicit clustering step to improve diversity, e.g. \cite{Xu:2007,Nguyen:2004,Zhou:2009}. Clustering methods impose an unsupervised notion of diversity on the problem, which can hurt performance.  The \emph{active learning with small pools} is a variant of margin sampling that only considers a small pool of the larger candidate set \cite{Bottou:2007}. The motivation for this was to decrease computation, but it should also increase the diversity of selected points. Unfortunately, the method did not provide better experimental results than standard margin sampling (though it is faster, as intended) \cite{Bottou:2007}.

We will show that min-margin is able to consistently outperform other methods, including margin, by leveraging useful diversity in an intelligent way. Moreover, we are able to incorporate this added diversity while still being $O(|\mathcal{Z}|)$ in the candidate set size $|\mathcal{Z}|$.


This paper is organized as follows:
\begin{itemize}
\vspace{-0.2cm}
    \item In Section~\ref{sec:min_margin}, we introduce min-margin (Algorithm~\ref{alg:min_margin}). We give a demonstration on a Gaussian simulation and provide detailed intuition for how and why min-margin works.
    \item In Section~\ref{sec:theory}, we give a theoretical analysis on a simplified setting showing that min-margin outperforms margin sampling for sufficient batch size and candidate set size.
    \item In Section~\ref{sec:experiments}, we compare min-margin to a number of baselines on  MNIST, Fashion MNIST, benchmark UCI datasets, and two large real-world applications. We show that min-margin is a consistently good method across batch sizes.
\end{itemize}


\section{Min-Margin Active Sampling Algorithm}\label{sec:min_margin}
We introduce the min-margin algorithm and motivate it through a Gaussian simulation.
\subsection{Preliminaries}
We assume the standard set-up that one has an initial sample of $N$ labeled training examples $\mathcal{T}_0 = \{(x_i, y_i)\}$ for $i=1, \ldots, N$ where $x_i \in \mathbb{R}^D$ and $y_i \in \mathbb{N}$. Let $N_g = \sum_{i=1}^N \1_{y_i = g}$ be the number of initial examples labeled class $g$, for classes $g \in \mathbb{N}$.  We also assume a candidate sample set $\mathcal{Z}$ of examples $z \in \mathbb{R}^D$ that can be chosen to be labeled. The goal is to select $B$ candidate examples from $\mathcal{Z}$ to be labeled. 

Given a training set $\mathcal{T}$, a learning procedure $H$ returns a model $h := H(\mathcal{T})$, where $h(z;g)$ is the $g$th class discriminant function \cite{HTF} for an input $z \in \mathbb{R}^D$. For example, if the learner was a neural network, $h := H(\mathcal{T})$ would produce a trained neural network and $h(z;g)$ could be the softmax probability of class $g$ for example $z$.

We build on margin sampling \cite{Scheffer:2001}, which is a popular and well-regarded active learning strategy that selects candidates with the smallest margin, where the margin is defined as follows.

\begin{definition}[Margin (\emph{Scheffer et al. 2001} \cite{Scheffer:2001})]
\begin{equation*}
\textrm{margin}(h, z) = h(z; \hat{y}_1(z)) - h(z; \hat{y}_2(z))
\end{equation*}
where $\hat{y}_1(z)$ and $\hat{y}_2(z)$ are the highest and second-highest scoring classes under predictor $h$: 
\begin{equation*}
\hat{y}_1(z) = \arg \max_{g}  h(z; g) \hspace{10mm} \textrm{ and } \hspace{10mm} \hat{y}_2(z) = \arg \max_{g, g \neq \hat{y}_1(z)} h(z; g).
\end{equation*}
\end{definition}

\subsection{Min-margin Algorithm}

We propose a new active sampling method we term \emph{min-margin} in Algorithm \ref{alg:min_margin} that trains multiple models on stratified bootstrap samples of the labeled training data \cite{EfronTibshirani:1993}, then selects candidate examples that have the smallest margin of \emph{any} of the bootstrapped models.  Algorithm \ref{alg:min_margin} includes two hyperparameters: the number of bootstrapped models $K$ and the bootstrap sample size fraction $\beta$. We recommend setting these bootstrap values to traditional bootstrap defaults of $K=25$ bootstraps and $\beta=1$ \cite{EfronTibshirani:1993}, which are the settings we fix throughout all of our experiments (except those where we show the impact of those two hyperparameters).

\begin{algorithm}[t]
   \caption{Min-Margin Active Sampling}
   \label{alg:min_margin}
\begin{algorithmic}
   \STATE {\bf Inputs}: Initial sample $\mathcal{T}_0$, candidate sample set $\mathcal{Z}$, number of bootstrapped models $K$, bootstrap sample size fraction $\beta$, number of candidate examples to select $B$, learning procedure $H$ 
   \STATE {\bf Bootstrap}: For each $k = 1, \ldots, K$, let $\mathcal{T}_k$ be a random sample with replacement from $\mathcal{T}_0$ of $\lfloor \beta N_g \rfloor$ examples from each class $g$, and $h_k := H(\mathcal{T}_k)$
   \STATE {\bf Score}: For each candidate $z \in \mathcal{Z}$, let $score(z) := \displaystyle \min_{k \in [K]} \textrm{margin}(h_k, z)$. 
   \STATE {\bf Return} the $B$ candidates from $\mathcal{Z}$ with lowest $score$. 
\end{algorithmic}
\end{algorithm}

\subsection{Gaussian Simulation}

We illustrate how min-margin works and compare to margin and committee sampling in Fig.~\ref{fig:simulation} on a standard two-dimensional Gaussian simulation. We draw a small initial seed sample, as shown in Fig.~\ref{fig:simulation} Top-Left, along with the decision boundary of a model trained on the initial sample and the true decision boundary.  We draw a candidate set of 8000 samples to select from and test on 10000 samples.
\begin{figure}[h]
  \begin{center}
    \includegraphics[width=0.45\textwidth]{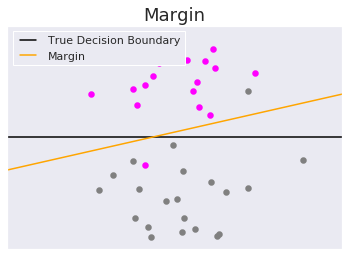}
    \includegraphics[width=0.45\textwidth]{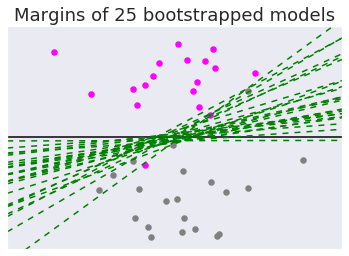} \\
    \includegraphics[width=0.45\textwidth]{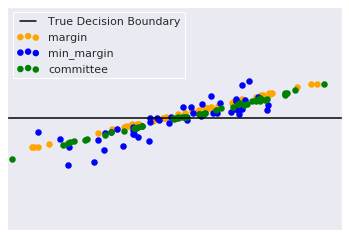} 
    \includegraphics[width=0.45\textwidth]{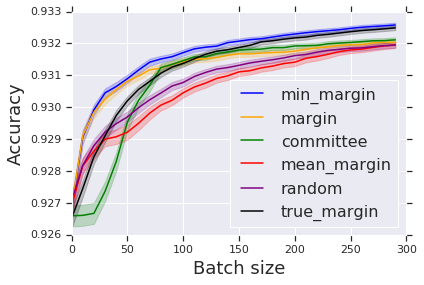} 
    \end{center}
  \caption{{\bf Illustrative Gaussian simulation}: Training examples for each class are generated from a two-dimensional Gaussian distribution with identity covariance but different means, and the learner is logistic regression. {\bf Top-Left}: The  initial random draw of $40$ training samples is shown, along with the true margin (black) and the decision boundary (yellow) for a logistic regression trained on the initial sample, which is used for the standard margin active sampling. {\bf Top-Right}: The decision boundaries of the $25$ models trained on the bootstrapped samples (green dashed lines); these decision boundaries are used by the committee and proposed min-margin active sampling. {\bf Bottom-Left}:  The first $50$ selected examples chosen from the $8000$ candidates by standard margin, QBC, and the proposed min-margin. {\bf Bottom-Right}: Test accuracy after re-training with selected samples added to the train set, plotted for different batch sizes, averaged over $500$ random draws of the train data, candidate set, and test data. The standard deviation is shaded. Min-margin is the best method for all batch sizes.}
\label{fig:simulation}
\end{figure}

Margin sampling would query points that are very close to the decision boundary produced by training on the initial sample, shown in yellow in Fig.~\ref{fig:simulation}  Top-Left. However, this decision boundary can be highly biased. Thus, margin sampling can fail to choose points in large critical regions necessary to learn the true decision boundary, even if given a large batch size $B$. Indeed, one sees in the Bottom-Left that margin queries points in a very narrow band defined by the initial biased decision boundary. The larger the candidate set $\mathcal{Z}$ is relative to $B$, the more margin sampling will suffer from this effect.

We note that committee algorithms that score based on a notion of maximal disagreement will have a similar problem of lacking diversity. Fig.~\ref{fig:simulation} Top-Right  shows  25 decision boundaries from the 25 models trained on 25 bootstrap samples.  Fig.~\ref{fig:simulation} Bottom-Left shows the points (green) selected by QBC to maximize the disagreement of the classification decisions of the 25 models \cite{Abe:1998}. There is a relatively narrow bi-conical region of the feature space where the 25 classifiers maximally disagree, and again the larger the candidate set $\mathcal{Z}$, the less diverse the green points will be, because with a large enough candidate set they can all occur within the narrow bi-conical region defined by maximal classifier disagreement. 

In contrast, the proposed min-margin method actively samples around each of the 25 bootstrapped margins, enabling it to sample throughout the region defined by the multiple margins shown in Fig.~\ref{fig:simulation} Top-Right.  Fig.~\ref{fig:simulation} Bottom-Left confirms that the min-margin selected points (blue) are more diverse. In this way, min-margin takes better advantage of the diversity of the set of bootstrapped models, and thus can choose better samples even if the initial sample was insufficient to learn a reasonable initial margin. That is, min-margin can leverage multiple weak learners for better active sampling. Fig.~\ref{fig:simulation} Bottom-Right confirms that min-margin's increased  diversity orthogonal to the initial margin adds value, with min-margin achieving the highest test accuracy for all tested batch sizes (plots shown were averaged over 500 random draws of the simulation). For small batches, standard margin is equally good, but for larger batches the extra diversity of min-margin is important. 

Fig.~\ref{fig:simulation} Bottom-Right also compares to pure random sampling, and to sorting by the margin of the mean of the bootstrapped models. This mean-margin approach performs poorly, which motivated us to also compare to the \emph{oracle} approach of sorting the candidates with respect to the \emph{true margin} to the true decision boundary (which is available for a Gaussian simulation, but unrealistic in practice). Remarkably, min-margin even outperforms margin sampling with the true margin, and is the only method to beat the true margin for all batch sizes. The relatively poor performance of true margin for small batch sizes is because the initial sample $\mathcal{T}_0$ still dominates the training and the model most needs examples that can correct the initial samples' confusions,  whereas in the large batch setting the true margin selects enough candidates to delineate the correct decision boundary regardless of the initial sample.

In fact, a remarkable property of the added diversity of min-margin is that it {\it adapts} to the uncertainty of the initial sample. If the initial sample is noisy, then the initial model will likely have high bias, but then the set of bootstrapped models will also likely have high variability, which makes it more likely that the bootstrap models collectively cover the region needed to improve learning. If the initial sample was sufficient to train a reasonable model, then it is likely that less diversity is necessary, and also that the bootstrapped models will have less variability.

\section{Theoretical Analysis}\label{sec:theory}

In this section, we provide theoretical insights for why min-margin can outperform margin. We note that margin-based sampling is known to be quite difficult to analyze: there are only a few prior works, and these assume the simplified setting of a uniform distribution on the unit ball with a linear classifier that passes through the origin \cite{balcan2007margin,wang2016noise}. We make similar assumptions, and further restrict to the two-dimensional linearly separable setting on the unit half-sphere.

Our result, Theorem~\ref{theorem:linear_svm}, says that as long as the candidate set and batch size are sufficiently large, and the min-margin procedure uses a diverse enough set of learned models (i.e. all possible models obtained from stratified bootstrap samples), and using a linear SVM as the learner, min-margin sampling will choose samples closer to the true decision boundary than margin sampling. In other words, min-margin does a better job at finding the more informative examples in the linearly separable case than margin sampling. As a reminder, a linear SVM in the linearly separable case in the binary classification setting would learn a linear decision boundary which correctly separates the two classes and maximizes the minimum distance of any example to decision boundary.

\begin{theorem}\label{theorem:linear_svm} Suppose that the data is uniformly distributed over the unit half-sphere in $\mathbb{R}^2$ with non-negative $x$-coordinate values (i.e. $ \{(\cos(\theta), \sin(\theta))\in \mathbb{R}^2 : \theta \in \left[-\frac{\pi}{2}, \frac{\pi}{2}\right] \}$) with two classes which are separable by the $x$-axis. Let $\mathcal{T}_0$ be the initial sample with $m$ samples from each class and $\mathcal{Z}$ be the candidate set. Let the learner be a linear SVM whose decision boundary passes through the origin.
 Suppose that Algorithm~\ref{alg:min_margin} uses $\beta \ge \frac{1}{m}$ and the learned models $\{h_1,...,h_K\}$ span all possible models obtained from a stratified bootstrap sample.
Suppose that:
\begin{align*}
    |\mathcal{Z}| \ge 12B\cdot m, \hspace{0.2cm} B \ge 3m^2\log(m), \hspace{0.2cm} m \ge 100.
\end{align*}
For samples $S \subseteq \mathcal{Z}$, define $\Theta^*(S) := \min \{ \theta^2  : (\cos(\theta), \sin(\theta)) \in S, \hspace{0.1cm} \theta \in [-\pi/2,\pi/2]\}$. That is, $\Theta^*(S)$ is the square of the angle w.r.t. the positive $x$-axis of the example in $S$ closest to the true decision boundary.
Let $S_{\text{margin}}, S_{\text{min-margin}} \subseteq \mathcal{Z}$ be the batches chosen based on margin and min-margin sampling, respectively. Then,
\begin{align*}
    \E[\Theta^*(S_{\text{min-margin}})] < \E[\Theta^*(S_{\text{margin}})].
\end{align*}
\end{theorem}

The proof of Theorem~\ref{theorem:linear_svm} is involved and is in the Appendix. Although Theorem~\ref{theorem:linear_svm} is in a simplified setting, it nonetheless provides insights into when min-margin has an advantage over margin and other similar methods. It confirms the intuition that with a small initial sample a learner trained can have high bias, and thus margin can perform poorly due to this bias while min-margin will not suffer from this effect as much as long as the batch size is large enough to leverage the diversity in the bootstrapped margins.
\section{Experiments}\label{sec:experiments}
We compare to margin sampling ({\bf margin}) and random sampling ({\bf random}), and to three alternate ways to use bootstrapped models: (i) the practical committee algorithm of Abe and Matmitsuka \cite{Abe:1998} ({\bf committee}), (ii) scoring by the variance in the softmax scores ({\bf var-softmax}), (iii) scoring by the mean of the bootstrapped margins ({\bf mean-margin}). The first round of experiments also has comparisons to the method of Brinker \cite{Brinker:2003} which balances the margin score with a pairwise diversity term ({\bf balanced-margin}),  $k$-centers coresets \cite{sener2017active} ({\bf k-centers}), and a half-half mixture of random and margin sampling ({\bf random-margin-mix}). More details about these methods can be found in the Appendix.


All of our experiments are for the one-shot setting; we only consider a single batch, rather than  the sequential setting where there are multiple batches and the model is retrained after each batch. We plot performance for different batch-sizes.

\subsection{MNIST and Fashion MNIST Experiments}
We test the performance on the MNIST and Fashion MNIST datasets.   We use an initial sample size of $100$ from the respective standard training set, and actively sample batches up to 2000 examples. Results in Fig.~\ref{fig:mnist} were averaged over $100$ different random splits between initial and candidate samples and test using the standard testing set. For MNIST, we use a neural network with two hidden layers each with 512 units and ReLU activations. For Fashion MNIST, we use the basic neural network used in the Keras tutorial, which has a hidden layer with 128 units. For all experiments, we train using the ADAM optimizer under default settings and train for $100$ epochs. The results in Fig.~\ref{fig:mnist} show that min-margin is the best strategy on both datasets, and again shows larger gains over margin for the larger batch sizes.

\begin{figure}[h]
  \begin{center}
  \begin{tabular}{ll}
    \includegraphics[width=0.48\textwidth]{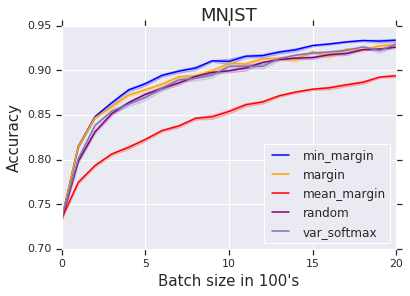} &
    \includegraphics[width=0.48\textwidth]{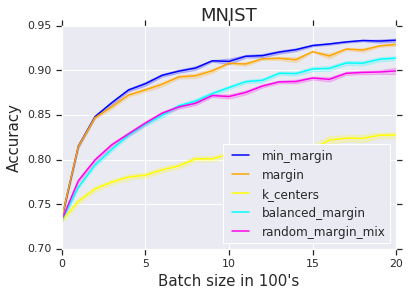} \\
    \includegraphics[width=0.48\textwidth]{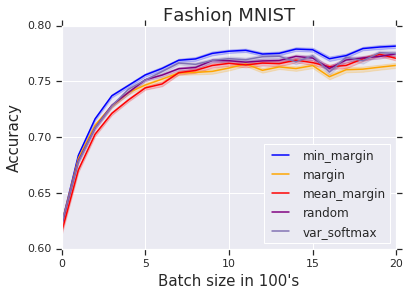} & 
    \includegraphics[width=0.48\textwidth]{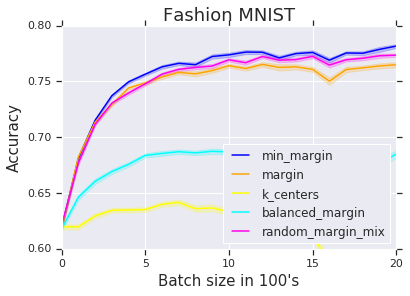}
        \end{tabular}
    \end{center}
  \caption{{\bf MNIST and Fashion MNIST Results}:  Results are shown averaged over $100$ random initial/candidate/test splits of the data with standard error bands shaded. Shown are the test accuracies vs. the batch size of the selected active samples. For both datasets, min-margin is one of the best performing methods for every batch size. The left and right plots for each datasets compare against different baselines.}
	\label{fig:mnist}
\end{figure}

\subsection{Benchmark Datasets}
We now show the performance of min-margin against the baselines on benchmark UCI datasets that are commonly used for testing active sampling methods: nomao (118 features, 34,465 datapoints), shuttle (9 features, 43,500 datapoints), magic04 (10 features, 19,020 datapoints), a9a (123 features, 32,561 datapoints) and cod-rna (8 features, 59,535 datapoints). Due to space, full results are in the Appendix.

For this set of experiments, we use logistic regression as the learner, trained using scikit learn's implementation \cite{scikit-learn} under default settings. We ran the experiment on $100$ random initial/candidate/test splits of each dataset, where the initial sample size is 100 and candidate/test sets are of equal size, and compute the mean test accuracies averaged over the runs. 

The results are shown in Fig.~\ref{fig:benchmark}. We see that min-margin performs competitively across datasets and batch sizes. Margin performs consistently well for small batch sizes, but can be quite poor with larger batch sizes, as seen in  both shuttle and magic04. In contrast, the proposed min-margin performs competitively with the best performer at every batch size on all four datasets. 

\begin{figure}[h]
\begin{center}
\begin{tabular}{ll}
    \includegraphics[width=0.48\textwidth]{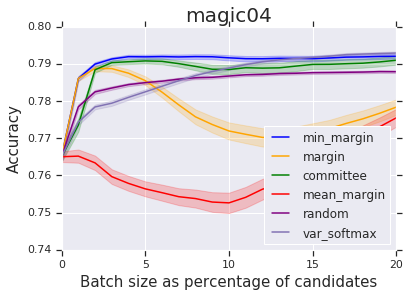}  & 
    \includegraphics[width=0.48\textwidth]{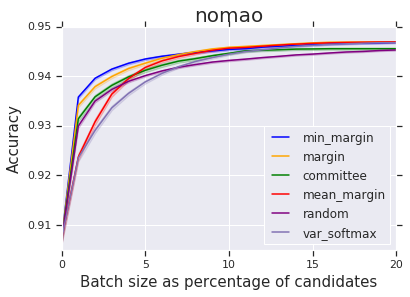} \\
    \includegraphics[width=0.48\textwidth]{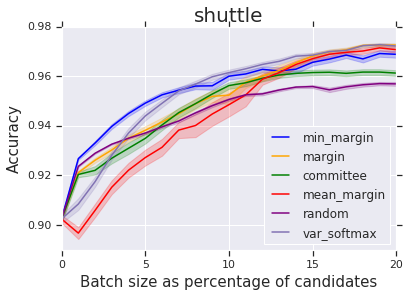} &
    \includegraphics[width=0.48\textwidth]{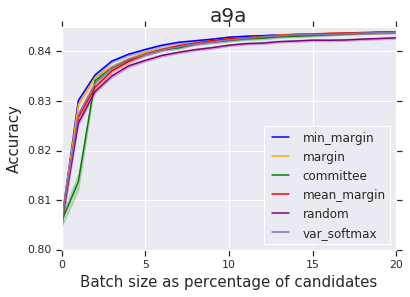}
    \end{tabular}
    \end{center}
\caption{{\bf Benchmark datasets}: The learner is logistic regression and results are shown averaged over $100$ random initial/candidate/test splits of the data with standard error bands. Shown are the test accuracies vs. the batch size of the selected active samples and the standard error is shaded.}
	\label{fig:benchmark}
\end{figure}

\subsection{Robustness To Hyperparameters}
The results in this paper are all hyperparameter tuning-free except for this section, in which we investigate how the min-margin performance changes if we change the number of bootstrap models $K$ while keeping the bootstrap sample size fixed to be the size of the initial training sample data (Fig.~\ref{fig:tuning} top row), or if we change the size of the bootstrap sample relative to the size of the dataset $\beta$ while keeping the number of bootstrapped models fixed to $K=25$ (Fig.~\ref{fig:tuning} bottom row).  The results in Fig.~\ref{fig:tuning} top row suggest that as long as the number of bootstrapped models is sufficiently large, the performance of min-margin is not very sensitive to this setting. All other experiments in this paper use a default $K = 25$.

On the other hand, Fig.~\ref{fig:tuning} shows one can improve the performance of min-margin by tuning the size of the bootstrap sample via $\beta$. A smaller bootstrap sample will generally give rise to more diverse bootstrapped models, resulting in give active sampling that is closer to random sampling - the smaller bootstrap size works best on shuttle, which is the dataset that random sampling works relatively best on. Using a larger bootstrap sample size should in general produce less diverse models, and thus it is not surprising that on magic04 the largest bootstrap sample size (dark blue) acts most like margin, with its poorer results for large batch sizes.   All other experiments in this paper use a default $\beta = 1$. 

\begin{figure}[h]
\begin{center}
\begin{tabular}{lll}
\includegraphics[width=0.31\textwidth]{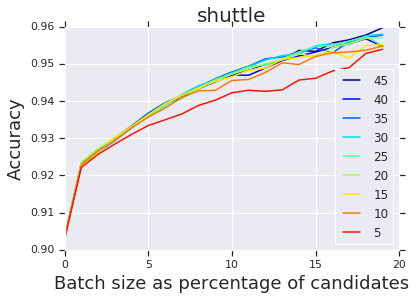}  & 
\includegraphics[width=0.31\textwidth]{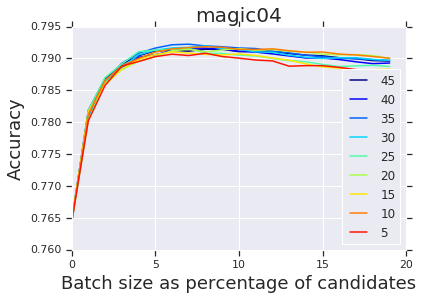}  & 
\includegraphics[width=0.31\textwidth]{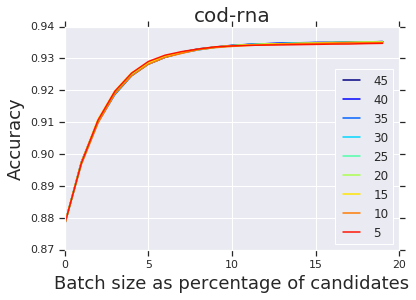} \\ 
\includegraphics[width=0.31\textwidth]{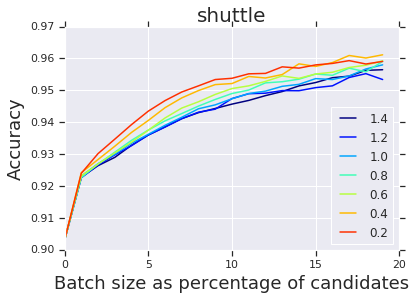} & 
\includegraphics[width=0.31\textwidth]{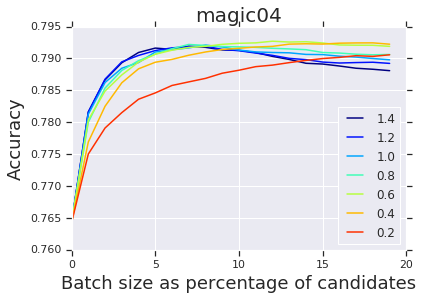}   & 
\includegraphics[width=0.31\textwidth]{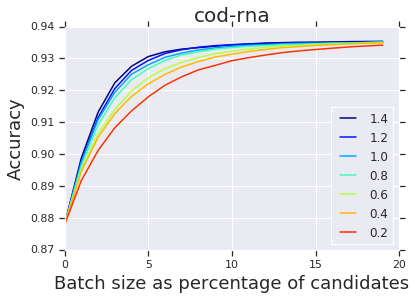} 
\end{tabular}
\end{center}
\caption{{\bf Top Row: Performance across number of bootstrapped models $K$}: We see that min-margin's performance is stable in the number of models used as long as the number of models is sufficiently high. {\bf Bottom Row: Performance across size of bootstrap sample as a fraction of original dataset $\beta$}: We see that there are opportunities to improve the performance of min-margin by adjusting the size of the bootstrap sample. For example, smaller bootstrap size is better for shuttle while larger bootstrap size is better for cod-rna and for magic04 there appears to be a trade-off between small and large batch size performance.}
  \label{fig:tuning}
\end{figure}

\subsection{Result Filtering and Query Intent Experiments}
In this section, we show the performance of the proposed min-margin on two proprietary binary classification datasets from a large internet services company. As in the other experiments, we limit the candidate set to be a subset of already-labeled examples, but in practice the true candidate set size $|\mathcal{Z}|$ would be orders of magnitude larger, increasing the importance of ensuring diversity.

\noindent \textbf{Results Filtering}: The task is to classify whether a candidate result is promising enough to be worth more expensive processing. This dataset has 16 features and 1,282,532 labeled examples.

\noindent \textbf{Query Intent}: The task is to classify the intent of a query. The dataset has 32 features and 420,000 examples. 

For both experiments we used a $2$-layer neural network as the learner with $10$ hidden units, trained using the ADAM optimizer with default settings and $20$ train epochs with mini-batch sizes of $100$. We average over $100$ random initial/candidate/test splits of the data where the initial sample size is 5000 training examples, and the rest of the data was split evenly between the candidate set and the test set. We show in the Appendix how results change as we vary the initial sample.

We show the results in Figure~\ref{fig:real_world} for 1,000 to 50,000 active sample batch sizes. Again, one sees that margin sampling's performance relative to random sampling degrades as batch sizes increase (x-axis on all four plots), especially as one moves to the larger batch sizes (bottom row). In contrast,  min-margin outperforms all other methods on the both problems for the larger batch sizes, while remaining competitive on the smaller batch sizes. We don't show the results for the other baselines as they were computationally infeasible for these large datasets. 
\begin{figure}[H]
  \begin{center}
    \includegraphics[width=0.48\textwidth]{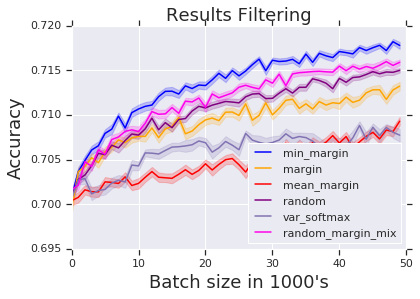}
    \includegraphics[width=0.48\textwidth]{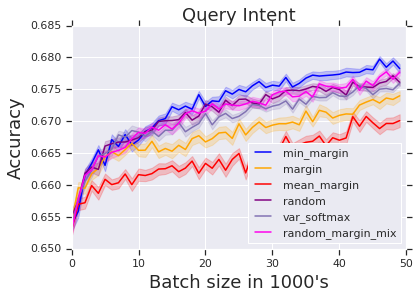}
    \end{center}
  \caption{{\bf Left}: Results Filtering. {\bf Right}: Query Intent.  Shown are the test accuracies vs. the additional batch size chosen and the standard error is shaded. We see that min-margin outperforms all other methods for Result Filtering and is the only method to do as well as random sampling for Query Intent on both the small and large batch situations.}
	\label{fig:real_world}
\end{figure}

\section{Conclusion and Discussion}
We have extended margin sampling to take the minimum margin over multiple bootstrapped models. This approach increases the diversity for better performance at batch active sampling. Experimentally across a broad set of problems, we have shown the proposal is competitive with margin sampling, and for larger batch sizes can be substantially better. Compared to other diversity-increasing strategies such as clustering candidates or penalizing candidate similarity, the proposed method is much more computationally efficient because it scales linearly in the candidate set size and training set size and is parallelizable. Further, the proposed method achieves its diversity by extending the core idea of margin sampling, rather than trading it off with an explicit diversity objective, and naturally adapts the amount of diversity to the model uncertainty. 

\clearpage
\section*{Acknowledgements}
We thank Giulia DeSalvo, Afshin Rostamizadeh, and Tim Hesterberg for helpful discussions.
{
\bibliography{ref}
\bibliographystyle{plain}
}

\clearpage
{
\appendix
\onecolumn
{\Large \bf Appendix}

\section{Proof of Theorem~\ref{theorem:linear_svm}}
We begin with a sketch of the proof. The set-up for the theorem is illustrated in Fig. \ref{fig:proof}. Recall that the initial decision boundary for the linear SVM will be midway between the positive and negative initial training examples closest to the x-axis, in Fig. \ref{fig:proof} this initial decision boundary is defined by the training examples $x_1$ and $x_6$. This initial decision boundary will have some bias.

The three key steps of the proof are as follows. (1) We first show that margin sampling will choose points within some small region of the initial decision boundary (marked by the blue square in the figure), and that this region will be limited in size by $B/\| \mathcal{Z} \|$, that is, the larger the batch size $B$ is the larger the blue-region will be, but the larger the candidate set size $\| \mathcal{Z} \|$, the smaller the blue region will be. (2) Next, we show that on average there will exist a bootstrap model that produces a decision boundary with less bias. For example, in Fig. \ref{fig:proof} the bootstrap decision boundary produced by a bootstrap sample that does not have $x_6$ would produce the less-biased bootstrap decision boundary shown, defined by the midpoint of examples $x_1$ and $x_7$. (3) Finally, the min-margin sampling technique will on average select some samples that are closest to that bootstrap model and thus the samples found based on min-margin sampling will closer to the true decision boundary than that of margin sampling.

\begin{figure}[h]
  \begin{center}
  \begin{tabular}{c}
    \includegraphics[trim={3cm 4.5cm 3cm 3cm},clip,width=0.9\textwidth]{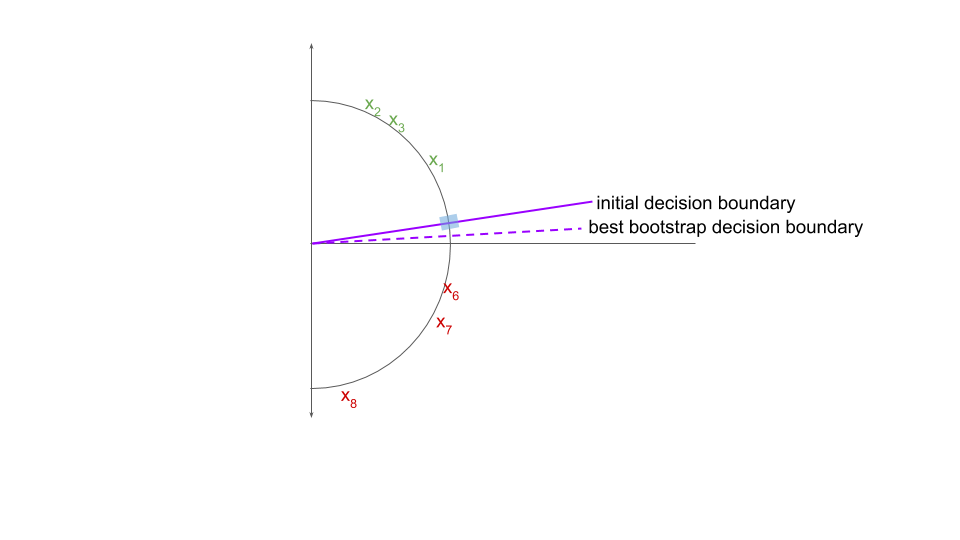}
        \end{tabular}
    \end{center}
  \caption{A toy illustration of the theorem set-up. Green denotes initial positive training examples, red denotes initial negative training examples. }
	\label{fig:proof}
\end{figure}

We first provide a couple of lemmas we will use later. 

\begin{lemma}\label{lemma:1} Suppose that $\theta_1,...,\theta_m$ are i.i.d. random variables drawn from a uniform distribution on $[0, 1]$.
Then, if $m \ge 20$, we have
\begin{align*}
    \E\left[\min_{j \in [m]} \left(\frac{1}{4} - \theta_j \right)^2\right] \le \frac{1.01}{2(m+1)(m+2)}.
\end{align*}
\end{lemma}

\begin{proof}
Suppose that $X = \left(\frac{1}{4} - \theta_1 \right)^2$. Then, we have the CDF of $X$ is as follows:
\begin{align*}
F_X(t) &= \mathbb{P}\left(\theta \in \left[\frac{1}{4} - \sqrt{t}, \frac{1}{4} + \sqrt{t}\right]\right) \\
&= 2\sqrt{t} \cdot \1\left[0 \le t < \frac{1}{16}\right] + \left(\frac{1}{4} + \sqrt{t}\right)\cdot  \1\left[\frac{1}{16}\le t < \frac{9}{16}\right] + \1\left[t \ge \frac{9}{16}\right].
\end{align*}
Then, we have that the density is
\begin{align*}
    f_X(t) = \frac{1}{\sqrt{t}} \cdot \1\left[0 \le t < \frac{1}{16}\right] + \frac{1}{2\sqrt{t}}\cdot \1\left[\frac{1}{16} \le t < \frac{9}{16} \right]
\end{align*}
Thus, we have that the density of $X_m := \min_{j \in [m]} \left(\frac{1}{4} - \theta_j \right)^2$ has density
\begin{align*}
    f_{X_m}(t) &= m\cdot f_X(t) \cdot (1 - F_X(t))^{m-1} \\
    &= \frac{m}{\sqrt{t}} \cdot (1 - 2\sqrt{t})^{m-1} \cdot \1\left[0 \le t < \frac{1}{16}\right] + \frac{m}{2\sqrt{t}} \cdot \left(\frac{3}{4} - \sqrt{t}\right)^{m-1} \cdot \1\left[\frac{1}{16} \le t < \frac{9}{16} \right].
\end{align*}
Thus,
\begin{align*}
\E\left[\min_{j \in [m]} \left(\frac{1}{4} - \theta_j \right)^2\right]
&= \E[X_m] = \int t \cdot f_{X_m}(t) dt \\
&= \int_{0}^{\frac{1}{16}} m\cdot \sqrt{t}\cdot (1 - 2\sqrt{t})^{m-1} dt + \int_{\frac{1}{16}}^{\frac{9}{16}} \frac{m\sqrt{t}}{2} \left(\frac{3}{4} - \sqrt{t}\right)^{m-1} dt \\
&= \frac{\frac{1}{2} + 2^{-m-4}(-m^2-5m-8)}{(m+1)(m+2)} +\frac{2^{-m-4} (m^2 + 7m + 16)}{(m+1)(m+2)} \\
&\le \frac{1.01}{2(m+1)(m+2)},
\end{align*}
as desired.
\end{proof}

\begin{lemma}\label{lemma:logbound}
Suppose that $B \ge 3m^2\log(m)$ and $m \ge 100$, then the following holds.
\begin{align*}
    \left(1 - \frac{1}{m^2} \right)^B \le \frac{1}{16(m+1)(m+2)}.
\end{align*}
\end{lemma}
\begin{proof}
We have using $m\ge 100$ and Taylor expansion of $\log$:
\begin{align*}
    B \ge 3m^2\log(m) \ge \frac{\log(17m^2)}{\frac{1}{m^2} + \frac{1}{2m^4} + \frac{1}{3m^6} + \cdots} = \frac{\log(17m^2)}{\log\left(1 - \frac{1}{m^2}\right)} = \frac{\log\left(\frac{1}{17m^2}\right)}{\log\left(1 - \frac{1}{m^2}\right)}.
\end{align*}
This implies that
\begin{align*}
    B \log\left(1 - \frac{1}{m^2}\right) \le \log\left(\frac{1}{17m^2}\right).
\end{align*}
Taking the exponential of both sides, we obtain
\begin{align*}
    \left(1 - \frac{1}{m^2} \right)^B \le \frac{1}{17m^2} \le \frac{1}{16(m+1)(m+2)},
\end{align*}
as desired.
\end{proof}

\begin{proof}[Proof of Theorem~\ref{theorem:linear_svm}]

Let $\theta_1,...,\theta_m$ and $-\tilde{\theta}_1,...,-\tilde{\theta}_m$ be the angles of the samples from the two classes, relative to the positive $x$-axis. We have that $\theta_i$ and $\tilde{\theta}_i$ are i.i.d. uniform from $[0, \frac{\pi}{2}]$ for $i \in [m]$.

Let $\theta_{(1)}$ and $\tilde{\theta}_{(1)}$ be the respective first order statistic of the angles for the two classes (i.e. $\theta_{(1)} = \min_{i\in[m]} \theta_i$ and $\tilde{\theta}_{(1)} = \min_{i\in[m]} \tilde{\theta}_i$).  Then, we have that $\theta_{(1)}$ and $\tilde{\theta}_{(1)}$ have distribution $\frac{\pi}{2} \cdot Beta(1, m)$ where $Beta(1, m)$ has density function $m(1 - x)^{m-1}\cdot \1[x \in [0, 1]]$ since the first order statistic of $m$ i.i.d. uniform distributions on $[0, 1]$ has $Beta(1, m)$ distribution. Then, we have that the decision boundary of linear SVM trained on $\mathcal{T}_0$, which we call $h_{\text{margin}}$ has angle $\frac{\theta_{(1)} - \tilde{\theta}_{(1)}}{2}$ w.r.t. the positive $x$-axis.

Let $\text{Bias}(h)$ denote the angle of the decision boundary of classifier $h$  w.r.t. the positive $x$-axis and define $ \mathcal{E}(h) := E_{\mathcal{Z}, \mathcal{T}_0}[\text{Bias}(h)^2]$. We thus have
\begin{align*}
    \mathcal{E}(h_{\text{margin}}) &= \E[ \text{Bias}(h_{\text{margin}})^2] = \E\left[\left(\frac{\theta_{(1)} - \tilde{\theta}_{(1)}}{2}\right)^2\right] \\
    &= \frac{1}{4}\left(\E[\theta_{(1)}^2] - 2 \E[\theta_{(1)}\cdot \tilde{\theta}_{(1)}] + \E[\tilde{\theta}_{(1)}^2]\right) \\
    &= \frac{\pi^2}{16}\left(\frac{2}{(m+2)(m+1)}-\frac{2}{(m+1)^2} +\frac{2}{(m+2)(m+1)} \right) \\
    &= \frac{\pi^2\cdot m}{8(m+1)}\cdot \frac{1}{(m+1)(m+2)}.
\end{align*}

Let $\mathcal{H}$ be the set of bootstrapped models. We have that the decision boundary for any $h \in \mathcal{H}$ is angle $\frac{\theta_i - \tilde{\theta}_j}{2}$ where $\theta_i$ is the smallest angle out of samples from the first class and $\tilde{\theta}_j$ is the smallest angle out of the samples from the second class. We now have
\begin{align*}
    \min_{h \in \mathcal{H}} \mathcal{E}(h) = \min_{h \in \mathcal{H}} \E[\text{Bias}(h)^2] &= \E\left[\min_{i, j \in [m]} \left(\frac{\theta_i - \tilde{\theta}_j}{2} \right)^2\right] \le 
    \E\left[\min_{j \in [m]} \left(\frac{\theta_{(\lfloor m/2\rfloor)} - \tilde{\theta}_j}{2} \right)^2\right],
\end{align*}
where $\theta_{(\lfloor m/2\rfloor)}$ is the $\lfloor m/2 \rfloor$-th order statistic of $\theta_1,...,\theta_m$. This is distributed as $\frac{\pi}{2} \cdot Beta(\lfloor m/2\rfloor, \lceil m/2\rceil)$. Therefore, $\E[\theta_{(\lfloor m/2\rfloor)}] = \frac{\pi}{2} \cdot \lfloor m/2\rfloor / m$ and $Var(\theta_{(\lfloor m/2\rfloor)}) = \frac{\pi^2}{4} \lfloor m/2\rfloor \cdot \lceil m/2\rceil / (m^2\cdot (m+1))$. Hence, using Chebyshev's inequality, we have
\begin{align*}
    \mathbb{P}\left[\left(\theta_{(\lfloor m/2\rfloor)} - \frac{\pi}{4}\right)^2 \ge \frac{\pi^2}{64}\right] \le 
    \frac{32}{m+1}.
\end{align*}
Therefore, 
\begin{align*}
     &\E\left[\min_{j \in [m]} \left(\frac{\theta_{(\lfloor m/2\rfloor)} - \tilde{\theta}_j}{2} \right)^2\right] \le 
     \E\left[\min_{j \in [m]} \left(\frac{\theta_{(\lfloor m/2\rfloor)} - \tilde{\theta}_j}{2} \right)^2  \middle| |\theta_{(\lfloor m/2\rfloor)} - \frac{\pi}{4}| \le \frac{\pi}{8} \right] \cdot  \left(1- \frac{32}{m+1}\right) 
    \\ &\hspace{3.5cm} + \E\left[\min_{j \in [m]} \left(\frac{\theta_{(\lfloor m/2\rfloor)} - \tilde{\theta}_j}{2} \right)^2  \middle| |\theta_{(\lfloor m/2\rfloor)} - \frac{\pi}{4}| > \frac{\pi}{8} \right] \cdot   \frac{32}{m+1} \\
    &\le \frac{1}{4} \cdot \E\left[\min_{j \in [m]} \left(\frac{\pi}{8} - \tilde{\theta}_j \right)^2 \right] \cdot  \left(1- \frac{32}{m+1}\right)  + \frac{1}{4} \cdot \E\left[\min_{j \in [m]} \tilde{\theta}_j ^2 \right] \cdot  \frac{32}{m+1}\\
    &=  \frac{1}{4} \cdot \E\left[\min_{j \in [m]} \left(\frac{\pi}{8} - \tilde{\theta}_j \right)^2 \right] \cdot \left(1- \frac{32}{m+1}\right)  + \frac{\pi^2}{8} \frac{1}{(m+1)(m+2)} \frac{32}{m+1} \\
    &\le \frac{\pi^2}{16} \frac{1.01}{2(m+1)(m+2)} \cdot \left(1- \frac{32}{m+1}\right)  + \frac{\pi^2}{8} \frac{1}{(m+1)(m+2)} \frac{32}{m+1} \le \frac{5\pi^2}{64(m+1)(m+2)}.
\end{align*}
where the second last inequality follows from Lemma~\ref{lemma:1}.

Next, the angles w.r.t. the positive $x$-axis of examples in $\mathcal{Z}$ are sampled uniformly in $[-\pi/2,\pi/2]$. Let $S_{\text{margin}}$ be the $B$ examples chosen from $\mathcal{Z}$ based on the margin.
Then under expectation, 
\begin{align*}
    \E[\Theta^*(S_{\text{margin}})] &\ge \mathcal{E}(h_{\text{margin}}) - \left(\frac{B\cdot \pi}{|\mathcal{Z}|}\right)^2 \ge \mathcal{E}(h_{\text{margin}})  - \frac{\pi^2}{128 (m+1)(m+2)} \\
    &\ge \frac{15\pi^2}{128 (m+1)(m+2)}  - \frac{\pi^2}{128 (m+1)(m+2)} \ge \frac{7\pi^2}{64 (m+1)(m+2)} ,
\end{align*}
where the second inequality follows from the condition that $|\mathcal{Z}| \ge 12Bm^2$.
Let $h_M := \argmin_{h \in \mathcal{H}} \text{Bias}(h)^2$ and let $A$ be the event that we sample a candidate from each class from $\mathcal{Z}$ that's closest to $h_M$ out of $\mathcal{H}$ under min-margin sampling. 
Then, we have
\begin{align*}
    \E[\Theta^*(S_{\text{min-margin}})] &\le \mathcal{E}(h_M) \cdot \mathbb{P}(A) + \frac{\pi^2}{4} \cdot (1 - \mathbb{P}(A)) + \left(\frac{B\cdot \pi}{|\mathcal{Z}|}\right)^2 \\
    &= \mathcal{E}(h_M) \cdot \left(1 - \left(1 - \frac{1}{m^2}\right)^{B}\right) +\frac{\pi^2}{4} \cdot (1 - \mathbb{P}(A)) + \left(\frac{B\cdot \pi}{|\mathcal{Z}|}\right)^2  \\
    &\le \mathcal{E}(h_M)\cdot \left(1 - \left(1 - \frac{1}{m^2}\right)^{B}\right)^2 + \frac{\pi^2}{4} \cdot (1 - \mathbb{P}(A))  + \frac{\pi^2}{64 (m+1)(m+2)}\\
    &\le \frac{5\pi^2}{64 (m+1)(m+2)} +  \frac{\pi^2}{64(m+1)(m+2)} +  \frac{\pi^2}{128 (m+1)(m+2)} \\
    &= \frac{13\pi^2}{128 (m+1)(m+2)},
\end{align*}
where the first inequality holds from the fact that under the event $A$, then we will choose a sample near $h_M$'s decision boundary and the angle of any sample is within $[-\pi/2,\pi/2]$ otherwise.
The third inequality holds by Lemma~\ref{lemma:logbound} and the bound on $\mathcal{E}(h_M) = \min_{h \in \mathcal{H}} \mathcal{E}(h)$ shown earlier.
It thus follows that
\begin{align*}
\E[\Theta^*(S_{\text{min-margin}})]< \E[\Theta^*(S_{\text{margin}})],
\end{align*}
as desired.
\end{proof}

\section{Additional Experiments}
\subsection{Additional baseline details}

All methods that use the bootstrapped models used the same set of bootstrapped classifiers. All compared methods are used under their default settings (for the bootstrap models, we used our recommended default of $K = 25$ bootstrap models, and $\beta = 1$ bootstrap fraction). We now give a description of the additional baselines used:
\begin{itemize}
    \item {\bf balanced-margin}. This is method of \cite{Brinker:2003} which adds to the objective a normalized term to encourage diversity. It selects a batch by iteratively choosing the candidate which minimizes $\lambda \cdot |g(x_i)| + (1 - \lambda) \max_{j \in S} k(x_i, x_j)$ where $S$ are the indices of the examples chosen thus far and adding it to the set $S$ until the batch size is reached and $k$ is a normalized distance kernel. We use their default setting of $\lambda = 0.5$ and use the cosine similarity as $k$.
    \item {\bf k-centers}. This is a method presented in \cite{sener2017active} which selects the batch using the greedy $k$-centers algorithm (i.e. setting $k$ to be the size of the batch and the centers would be the examples to query for labels). As a reminder, the $k$-centers objective is to find $k$ centers that minimizes the maximum distance from any example to its closest center. We use cosine similarity as the distance metric. We note that \cite{sener2017active} provides a robust version of $k$-centers which they show performs marginally better than $k$-centers. We use the basic version which has no hyperparameters to tune. We however note that \cite{sener2017active} designed this method for convolutional neural networks, while in our experiments, none were convolutional neural networks. 
    \item {\bf random-margin-mix}. This method chooses half the examples based on margin and the other half randomly.
\end{itemize}
We show additional charts comparing these methods. We found that these methods were often not competitive with many of the baselines shown in the main text. We show these for the benchmark datasets and MNIST and Fashion MNIST. For the real-world case studies, balanced-margin and k-centers were computationally infeasible due to the size of the datasets so we don't show the comparisons for those.

\subsection{Benchmark datasets}

\begin{figure}[H]
\begin{center}
    \includegraphics[width=0.45\textwidth]{figures/magic04}  
    \includegraphics[width=0.45\textwidth]{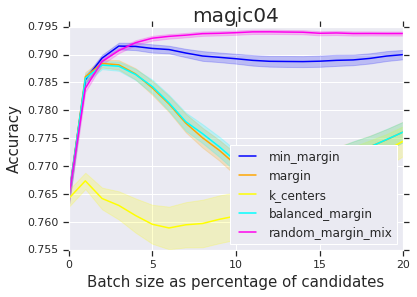}  
\end{center}
\end{figure}
\begin{figure}[H]
\begin{center}
    \includegraphics[width=0.45\textwidth]{figures/nomao}  
    \includegraphics[width=0.45\textwidth]{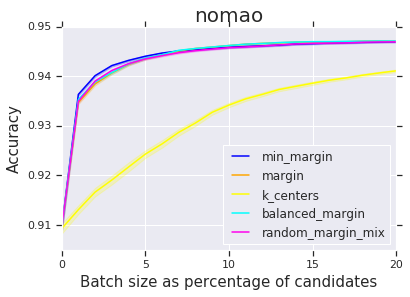}  
\end{center}
\end{figure}

\begin{figure}[H]
\begin{center}
    \includegraphics[width=0.45\textwidth]{figures/shuttle}  
    \includegraphics[width=0.45\textwidth]{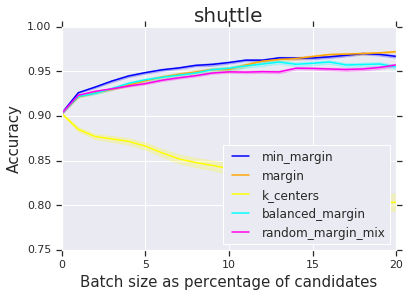}  
\end{center}
\end{figure}

\begin{figure}[H]
\begin{center}
    \includegraphics[width=0.45\textwidth]{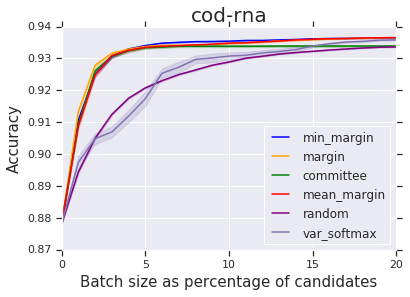}  
    \includegraphics[width=0.45\textwidth]{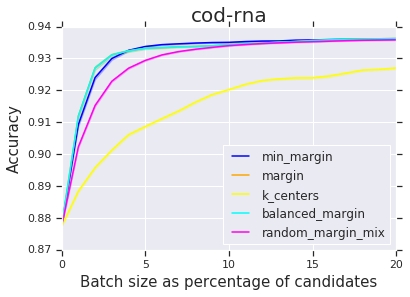}  
\end{center}
\end{figure}

\begin{figure}[H]
\begin{center}
    \includegraphics[width=0.45\textwidth]{figures/a9a}  
    \includegraphics[width=0.45\textwidth]{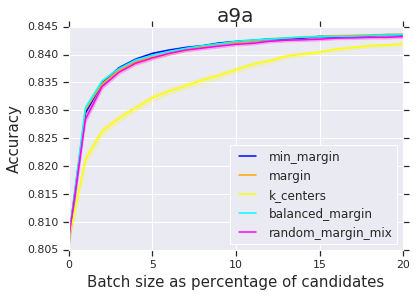}  
\end{center}
\end{figure}

\subsection{Real World Datasets}

\begin{figure}[H]
\begin{center}
    \includegraphics[width=0.31\textwidth]{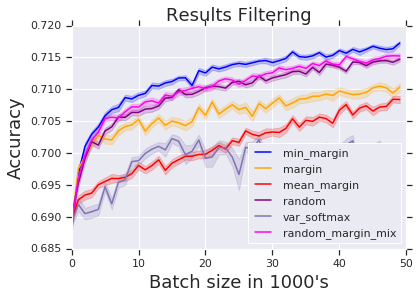}  
    \includegraphics[width=0.31\textwidth]{figures/filtering_5000}  
    \includegraphics[width=0.31\textwidth]{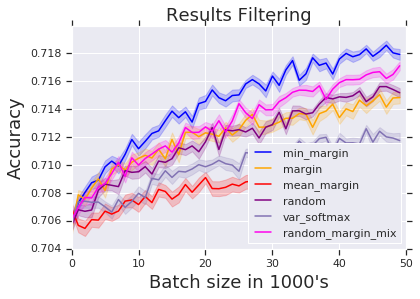}  
\end{center}
\caption{{\bf Results Filtering across initial sample sizes}. {\bf Left}: 2000 initial sample, {\bf Middle}: 5000 initial sample. {\bf Right}: 10,000 initial sample}
\end{figure}

\begin{figure}[H]
\begin{center}
    \includegraphics[width=0.31\textwidth]{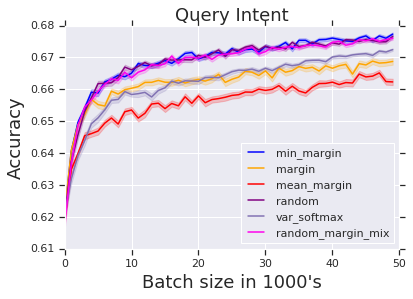}  
    \includegraphics[width=0.31\textwidth]{figures/query_intent_5000}  
    \includegraphics[width=0.31\textwidth]{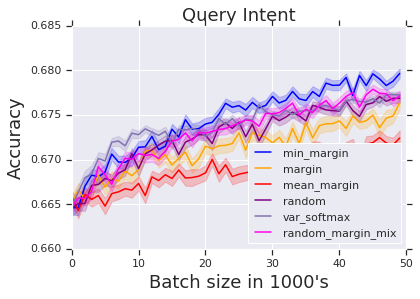}  
\end{center}
\caption{{\bf Query Intent across initial sample sizes}. {\bf Left}: 2000 initial sample, {\bf Middle}: 5000 initial sample. {\bf Right}: 10,000 initial sample}
\end{figure}
}

\end{document}